\documentclass{article}

 \usepackage[preprint, nonatbib]{neurips_2026}

\usepackage[square,numbers]{natbib}
\usepackage[utf8]{inputenc}
\usepackage{amsmath,amssymb,amsthm}
\usepackage{mathtools}
\usepackage{booktabs}
\usepackage{hyperref}
\usepackage{amsthm}

\usepackage{algorithm}
\usepackage{algorithmic}
\usepackage[utf8]{inputenc} 
\usepackage[T1]{fontenc}    
\usepackage{hyperref}       
\usepackage{url}            
\usepackage{booktabs}       
\usepackage{amsfonts}       
\usepackage{nicefrac}       
\usepackage{microtype}      
\usepackage{xcolor}         

\newtheorem{theorem}{Theorem}[section]
\newtheorem{lemma}[theorem]{Lemma}

\usepackage{graphicx}
\usepackage{multirow}
\usepackage{amsmath}
\usepackage{amssymb}
\usepackage{dsfont}
\usepackage{subcaption}
\DeclareMathOperator*{\argmax}{arg\,max}

\usepackage{wrapfig}

\usepackage{enumitem}
\usepackage{tcolorbox}
\usepackage{titlesec}
\tcbuselibrary{breakable, listings, skins}

\newtcblisting{promptbox}[2][blue]{%
  enhanced,
  breakable,
  toprule at break=0pt,
  bottomrule at break=0pt,
  pad at break=2mm,
  colback=#1!10,            
  colframe=#1!50!black,    
  colbacktitle=#1!30!white,
  boxrule=0.5pt,
  arc=0mm,
  left=2mm, right=2mm,
  top=-1mm, bottom=1mm,
  colbacktitle=#1!30!white,
  coltitle=black,
  fonttitle=\bfseries,
  toptitle=1mm,      
  bottomtitle=1mm, 
  title=#2,
  listing only,
  listing options={
    breaklines=true,
    breakatwhitespace=true,
    basicstyle=\ttfamily\small,
    columns=fullflexible
  }
}

\definecolor{softgreen}{rgb}{0, 0.6, 0.0}

\title{Budget-Efficient Automatic Algorithm Design via Code Graph}

\author{Maxime Bouscary\\
Operations Research Center\\
Massachusetts Institute of Technology\\
\texttt{mbscry@mit.edu}
\And
Manxi Wu \\
Department of Civil and Environmental Engineering \\
University of California, Berkeley \\
\texttt{manxiwu@berkeley.edu}
\And
Saurabh Amin \\
Laboratory for Information \& Decision Systems \\
Massachusetts Institute of Technology\\
\texttt{amins@mit.edu}
}

\begin{document}

\maketitle

\begin{abstract}
Large language models (LLMs) have emerged as powerful tools for automatic algorithm design (AAD). However, existing pipelines remain inefficient. They operate at the granularity of full algorithms, redundantly rewriting recurring substructures and discarding low-fitness candidates that may contain valuable algorithmic features. We formalize \emph{budget-efficient automatic algorithm design}, wherein the search policy maximizes realized fitness subject to limited computational cost. We propose a directed acyclic graph representation of algorithms and build a search framework that fully exploits the LLM's output. Instead of querying the LLM for full algorithms, we use it to obtain corrections: compact operators that add, replace, or remove code blocks. Each correction augments the graph, yielding new algorithms that compose with prior corrections. This graph structure decomposes algorithms into sets of corrections, enabling correction-level credit assignment that informs subsequent queries. We complement this framework with theoretical insights into the ideal balance between search depth and breadth at different budget levels. We validate our method empirically on three combinatorial optimization problems, demonstrating consistent superiority of our graph-based search over full-algorithm search at equal token budget. Finally, our experiments suggest that rich contexts help only when the LLM's prior knowledge is shallow, and can hinder performance otherwise.
\end{abstract}

\section{Introduction}

Automatic algorithm design (AAD) has transitioned from classical symbolic search to a paradigm where large language models (LLMs) serve as general-purpose search operators. While this shift has enabled notable successes in mathematical discovery \citep{romera2024mathematical, novikov2025alphaevolve}, heuristic design \citep{liu2024evolution, ye2024reevo, zheng2025monte}, and domain-specific algorithm engineering \citep{yazdani2025evocut, surina2025algorithm}, current pipelines remain remarkably wasteful, often requiring thousands or even millions of LLM queries. Existing methods typically generate a single, full algorithm per query. This full-algorithm paradigm gives rise to two fundamental limitations. 

First, generated algorithms are redundant: generated candidates share substantial substructure, yet recurring subroutines are rewritten from scratch in every query. Problems admitting short algorithms, such as bin packing or circle packing, are largely shielded from this overhead, but more complex combinatorial problems yield algorithms spanning hundreds of lines of code. The full-rewrite approach is not only computationally expensive but also error-prone, as it increases the likelihood of introducing syntactical or logical mistakes in otherwise valid code blocks \citep{wang2024large, dou2026wrong, cheng2026diff}.

Second, the resulting algorithms are opaque: the relations between algorithms are only implicit, and natural-language descriptions of algorithms are typically too coarse-grained to capture the main drivers of performance. The full-algorithm paradigm thus offers no principled way to extract the most impactful algorithmic features, and the LLM is forced to rely on costly and inconsistent natural-language reflections. These inefficiencies translate directly into wasted tokens at a time when LLM inference has become the central cost of AAD.

Such waste motivates a budget-efficient formulation of AAD that seeks to maximize fitness subject to an explicit cap on cumulative inference costs. This budget-efficient view exposes a fundamental trade-off that is largely implicit in prior work: while richer LLM queries (e.g., extensive context, structured feedback, or in-context examples) yield higher-quality candidate algorithms, they rapidly deplete the available budget, restricting the overall depth or breadth of the search. To the best of our knowledge, no search architecture has been explicitly designed to navigate this trade-off and optimize performance under strict cost constraints.

We introduce a search framework that addresses these challenges by navigating the algorithm space through a directed acyclic graph (DAG), where source-to-sink paths represent complete algorithms. Rather than requesting full programs, the system queries the LLM for concise corrections that add, replace, or remove specific code blocks. This approach allows a single query to return multiple corrections that can be composed with those of previous iterations, significantly increasing the set of accessible algorithms without additional LLM queries. Additionally, each algorithm is composed of an explicit set of corrections, which enables correction-level credit assignment. We use this granular information to construct structured feedback that guides subsequent iterations toward high-performing algorithms.

Our main contributions are as follows.
\begin{itemize}
    \item We formalize \emph{budget-efficient automatic algorithm design}~\eqref{eq:bcaad} and propose a graph-based representation of algorithms as source-to-sink paths in a graph of code blocks. Building on this representation, we introduce two budget-efficient AAD methods: a context-agnostic variant, and a context-guided variant that additionally conditions the LLM on a summary of correction credit estimates.
    \item We provide theoretical insights on configuring AAD frameworks under budget constraints, focusing on the trade-off between the depth and breadth of search. Under assumptions on the mean and dispersion of search trajectories, we show that (i) the set of locally optimal per-run budgets is non-decreasing in the total budget, with sufficient conditions for uniqueness; (ii) for sufficiently large budgets, the optimal search strictly balances depth and breadth, ruling out both pure zero-shot sampling and a single exhaustive run; (iii) this balance favors depth for policies whose expected fitness grows faster with the per-run budget.
    \item We empirically evaluate our framework on three combinatorial optimization problems with varying degrees of pre-training exposure. The graph-based search consistently dominates algorithm-based search at equal token budget. Further, our experiments reveal the importance of in-context learning when the LLM's prior is weak, while similar guidance can hurt performance when the prior is strong.
\end{itemize}

\section{Related Work}
 
LLM-based AAD has evolved around a few recurring architectural patterns. We organize this line of research into three categories: evolutionary methods, tree-search methods, and methods that fine-tune the underlying LLM.
 
\paragraph{Evolutionary methods.}
The dominant framework couples an LLM generator with a population-based selection. Candidate algorithms are sampled, evaluated on a fixed set of instances, and the best ones are kept to seed the next generation. FunSearch~\citep{romera2024mathematical} popularized this framework for mathematical discovery, and EoH~\citep{liu2024evolution, yao2025multi} refined it for general heuristic design by incorporating natural-language descriptions of heuristics into the generation prompts. Subsequent work has extended the loop along several axes: ReEvo~\citep{ye2024reevo} adds reflective self-critique between generations, EvoCut~\citep{yazdani2025evocut} targets cutting-plane generation for integer programming, and EvoTune~\citep{surina2025algorithm} and TIDE~\citep{chen2026tide} integrate online tuning of the generator into the evolutionary loop. These methods share a common limitation: each generation produces full algorithms, and population diversity is maintained only through prompt-level mechanisms, leaving the search prone to paradigm collapse when the LLM's prior is concentrated around a few algorithmic features. AlphaEvolve~\citep{novikov2025alphaevolve} and OpenEvolve~\citep{openevolve} operate at the level of code edits rather than full algorithms, substantially reducing inference cost on long programs. However, all edits are applied simultaneously and directly to their target algorithm, and only the resulting full algorithm is stored as a new candidate. In contrast, our corrections are toggleable and reusable, allowing us to yield a greater number of algorithms per LLM query.
 
\paragraph{Tree-search methods.}
A second popular approach organizes the algorithms produced during the search into a tree, where children represent LLM-proposed refinements of their parent. MCTS-AHD \citep{zheng2025monte} and Planning of Heuristics \citep{wang2025planning} apply Monte Carlo tree search to balance the exploitation of high-fitness branches against the exploration of underperforming ones. Similarly, CodeTree \citep{li2025codetree} uses agent-guided tree expansion for general code generation, while \citet{astorga2024autoformulation} tailors tree search to optimization model formulation. These methods preserve the full search history and avoid paradigm collapse, but each node remains a complete algorithm, and the difference between a parent and its children is obscure, making it impossible to isolate beneficial algorithmic features. Our graph structure overcomes this by making the corrections present in any algorithm explicit and observable. This enables correction-level credit assignment to better guide the search, a granularity unavailable when the unit of analysis is a complete algorithm.
 
\paragraph{Fine-tuning and data augmentation.}
A third thread focuses on improving the LLM itself rather than the search procedure. ORLM~\citep{huang2025orlm} and OptiMind~\citep{chen2025optimind} fine-tune LLMs on curated optimization-domain corpora, MILP-Evolve~\citep{li2024towards} generates a large, diverse set of MILP instances to strengthen LLM performance on this problem class, and \citet{lima2025toward} addresses the scarcity of high-quality training data through verifiable synthetic generation. These approaches complement our work, as a stronger generator inherently improves the search procedures built upon it. Building on these efforts, we demonstrate that the optimal, budget-efficient search strategy depends on the LLM's prior strength: extensive search methods compensate for weak priors but often undermine the effectiveness of strong ones.



\section{Problem Setting}

Let $\mathcal{P}$ be a combinatorial optimization problem defined over an instance space $\mathcal{Z}$ and a solution space $\mathcal{Y}$, with instances drawn from a distribution $\mathcal{D}_{\mathcal{P}}$ over $\mathcal{Z}$. Let $\mathcal{A}$ denote the space of algorithms $a: \mathcal{Z} \to \mathcal{Y}$, and let $f: \mathcal{Z} \times \mathcal{Y} \to \mathbb{R}$ be a scoring function. The fitness of an algorithm $a \in \mathcal{A}$, where higher values indicate better performance, is:
\begin{equation*}
    V(a) \triangleq \mathbb{E}_{z \sim \mathcal{D}_{\mathcal{P}}}\left[f(z, a(z))\right].
\end{equation*}
Since $\mathcal{D}_{\mathcal{P}}$ is typically unavailable, $V$ is estimated by the empirical mean:
\begin{equation*}
    \widehat{V}(a) \triangleq \frac{1}{m}\sum_{i=1}^m f(z_i, a(z_i)),
    \qquad z_1,\dots,z_m \sim\mathcal{D}_{\mathcal{P}},
\end{equation*}
computed on a fixed training set of $m$ instances.

Algorithms are produced by a pre-trained LLM acting as a stochastic generator. Let $\mathcal{X}$ denote the set of all possible sequences over the LLM’s discrete token vocabulary. Given a context $x \in \mathcal{X}$, the generator induces a conditional distribution $L(\cdot \mid x)$ over finite, non-empty subsets of $\mathcal{A}$. At each iteration $t = 1, 2, \dots$, a search policy samples a context $x_t \in \mathcal{X}$, prompting the generator to yield a stochastic set of candidate algorithms $A_t \sim L(\cdot \mid x_t)$. Each generated algorithm $a \in A_t$ is evaluated to obtain its empirical fitness $\widehat{V}(a)$. We let:
\begin{equation*}
     h_t \triangleq \bigcup_{s=1}^{t-1} \left\{\left(a, \widehat{V}(a)\right) \mid a \in A_s\right\}
\end{equation*}
be the search history at iteration $t$, and denote by $\mathcal{H}$ the space of all search histories. A search policy is a mapping $\pi: \mathcal{H} \to \Delta(\mathcal{X})$ returning a probability distribution over contexts given a search history, with $\Pi$ being the set of all such search policies.

Each iteration incurs a generation cost $c: \mathcal{X} \to \mathbb{R}_{>0}$, with $c(x) \geq c_0 > 0$, capturing that certain contexts induce higher computational costs. Given a total budget $B \geq c_0$, the \emph{Budget-Efficient Automatic Algorithm Design} problem is:
\begin{equation}
    \label{eq:bcaad}
    \pi^{\star}(B) \in \argmax_{\pi \in \Pi} \mathbb{E}_{\pi, L} \left[ z(\pi, B) \right]
\end{equation}
with
\begin{equation*}
    z(\pi, B) \triangleq  \max_{t \in [T^\pi], a \in A_t^\pi} \widehat{V}(a) \quad \text{and} \quad T^\pi \triangleq \max \left\{ s \in \mathbb{N} : \sum_{t=1}^{s} c\left(x_t^\pi\right) \leq B \right\},
\end{equation*}
where $z(\pi, B)$ is a random variable representing the maximum empirical fitness achieved across all algorithms generated under the trajectory induced by policy $\pi$ constrained by budget $B$.

\section{Graph-Based Search} \label{sec:graph-based-search}
 
We propose a search framework that navigates the algorithm space $\mathcal{A}$ by evolving a directed acyclic graph $G_t$. In this graph, edges represent code blocks, and source-to-sink paths represent complete algorithms. The graph is augmented by \emph{corrections}: modifications that add, replace, or remove code blocks, which translate into new sub-paths in $G_t$ enlarging the set of accessible algorithms.
 
Our search framework offers two major advantages:
\begin{itemize}
    \item Higher query yield. Standard methods return one algorithm per query. Our approach returns $m \geq 1$ corrections for a similar cost, augmenting the graph $m$ times. With $M_t$ being the current number of paths in $G_t$, this results in $k_t \in [m, M_t \cdot 2^m]$ new paths per LLM query. In the most conservative scenario where overlapping and incompatible corrections limit the yield to $k_t = m$, our method maintains an $m$-fold increase in efficiency.
    \item Explicit Credit Assignment. Since algorithms correspond to paths in $G_t$, the specific corrections composing them are explicit. This allows us to estimate each correction's contribution to the fitness and use this information to steer the LLM toward high-impact features via in-context learning rather than relying solely on the LLM's prior.
\end{itemize}

\subsection{Graph Code Representation}
 
We maintain a directed acyclic graph $G_t = (V_t, E_t)$ with a unique source and sink, and let a \emph{path} refer to a sequence of edges connecting the source to the sink. Each edge $e \in E_t$ is associated with a block of code and the correction that introduced it. Any path translates into an algorithm, obtained by concatenating the code blocks along the path. We denote by $P_a$ the path inducing algorithm $a$. The set of algorithms accessible as paths in $G_t$ is denoted $\mathcal{A}_t \subseteq \mathcal{A}$. We partition it into $\mathcal{A}_t^{\text{eval}}$ and $\mathcal{U}_t$, the subsets of evaluated and unevaluated algorithms, respectively.
 
A \emph{correction} is an operator affecting the semantics of an algorithm by introducing or removing algorithmic features. It consists of one or more syntactic code modifications. Formally, a correction is a mapping:
\begin{equation*}
    \phi : D_\phi \to \mathcal{A}, \qquad D_\phi \subseteq \mathcal{A},
\end{equation*}
defined on a domain $D_\phi$ of compatible algorithms (those containing the code blocks $\phi$ removes, replaces, or extends). Applying $\phi$ to $a \in D_\phi$ replaces sub-paths of $P_a$ with new ones. Incorporating the new nodes and edges augments the graph and expands the set of accessible paths.
 
Writing $E_\phi \subseteq E_t$ for the set of edges introduced by correction $\phi$, the edges of $G_t$ are partitioned by their introducing correction, so the set of corrections used by an algorithm $a$ is directly readable from its path:
\begin{equation*}
    \mathcal{C}(a) \triangleq \big\{\, \phi : P_a \cap E_\phi \neq \emptyset \,\big\}.
\end{equation*}
 
Every algorithm extracted from $G_t$ therefore carries an explicit record of the algorithmic features it incorporates, which makes the search interpretable and enables correction-level credit assignment.

\subsection{Search Process}
 
The complete procedure is outlined in Algorithm~\ref{alg:graph_search}. At each iteration, we generate a set of corrections and cycle through them to extend the graph, evaluate new candidates, and update the surrogate model with the newly collected information. Figure~\ref{fig:correction_diagram} illustrates how corrections are processed in this cycle.
 
\begin{figure}[h]
    \centering
    \includegraphics[width=0.95\textwidth]{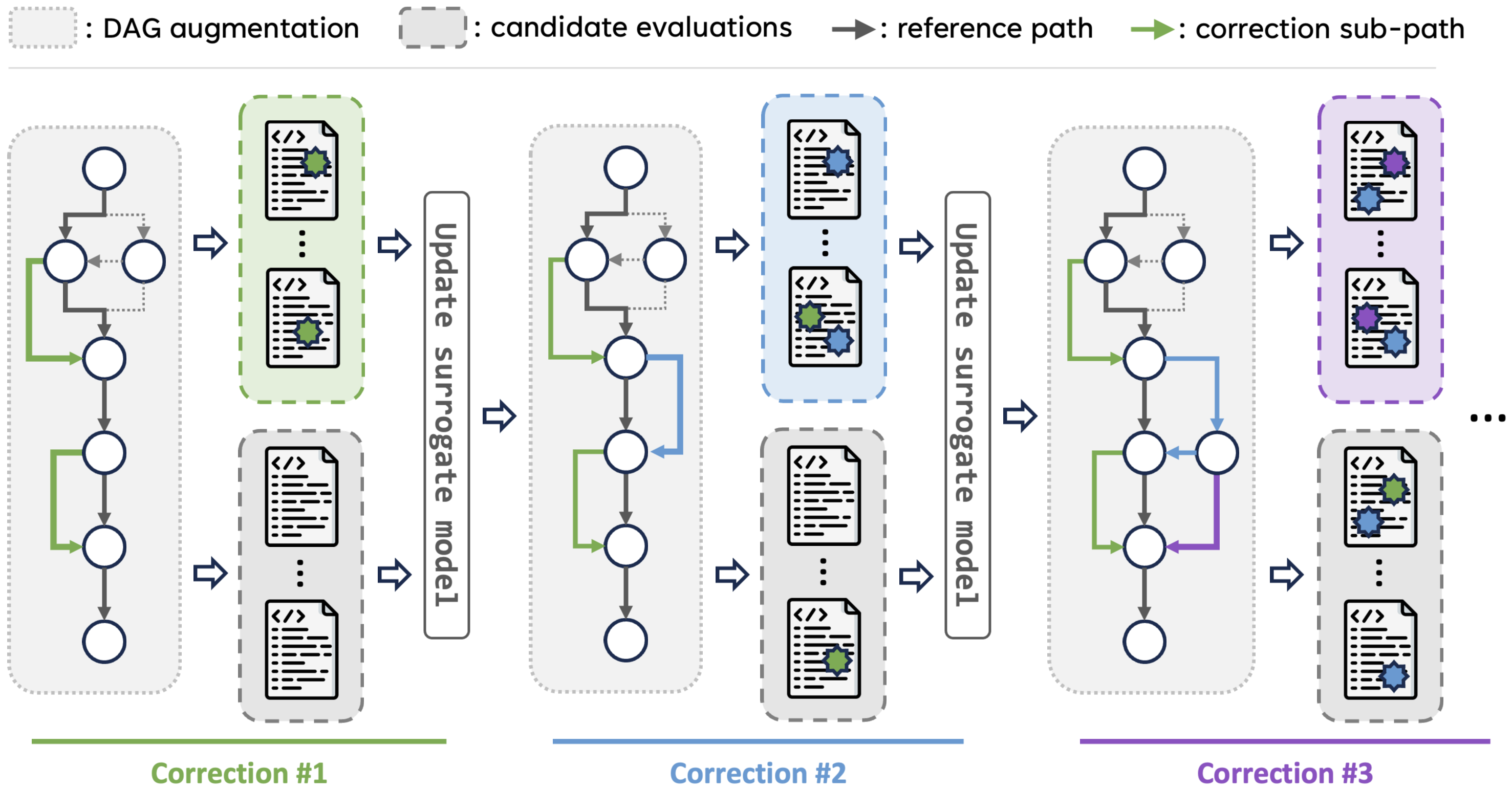}
    \caption{Single iteration of the graph-based search. Each iteration generates multiple corrections. A correction augments the graph by inserting sub-paths, enlarging the set of source-to-sink paths. From the augmented set of unevaluated paths $\mathcal{U}_t$, we form two evaluation groups: top-scoring algorithms re-routed through the new correction (colored group), and highest-ranked unevaluated paths under the surrogate model (uncolored group). Evaluated algorithms are added to the dataset, and the surrogate is retrained before applying the next correction.}
    \label{fig:correction_diagram}
\end{figure}

We propose two graph-based approaches: a context-agnostic and a context-guided variant. Below, we detail the core steps, alongside the supplementary step exclusive to the context-guided variant.

\paragraph{Generating corrections.} We select the highest-fitness evaluated algorithm as the reference algorithm:
\begin{equation*}
    a_{\textrm{ref}} \in \arg\max_{a \in \mathcal{A}_t^{\textrm{eval}}} \widehat{V}(a).
\end{equation*}
For the context-agnostic approach, we obtain a new set of corrections $\Phi_t$ by querying $\textrm{LLM}_\textrm{cor}$ conditioned solely on $a_{\text{ref}}$. The LLM analyzes the reference algorithm and proposes performance-improving corrections. The context-guided variant additionally provides $\textrm{LLM}_\textrm{cor}$ with a natural language summary, $\mathcal{S}_t$, synthesizing the empirical success of corrections applied in prior iterations.

\paragraph{Evaluating candidate paths.} The set of unevaluated paths $\mathcal{U}_t$ grows combinatorially with each correction, making exhaustive evaluation intractable. Therefore, when applying a new correction $\phi$ to the graph, we evaluate at most $n_{\text{eval}}$ candidates, split into two sets of equal size. The exploration set $\mathcal{E}_\phi^{(1)}$ probes the effect of the new correction $\phi$ at the high end of the fitness distribution. It consists of the highest-fitness algorithms compatible with $\phi$, with $\phi$ applied:
\begin{equation*}
\mathcal{E}_\phi^{(1)} = \left\{\, \phi(a) : a \in \mathrm{Top}^{\widehat{V}}_{n_{\text{eval}}/2}\left(\mathcal{A}_t^{\text{eval}} \cap D_\phi\right) \right\}.
\end{equation*}
The exploitation set $\mathcal{E}_\phi^{(2)}$ leverages the information accumulated over previous iterations, via the surrogate model $\mathcal{M}_t$, to concentrate computation on high-potential unevaluated algorithms:
\begin{equation*}
\mathcal{E}_\phi^{(2)} = \mathrm{Top}^{\mathcal{M}_t}_{n_{\text{eval}}/2}\left(\mathcal{U}_t\right).
\end{equation*}
Evaluations from both sets are appended to $\mathcal{A}_t^{\text{eval}}$ and $\mathcal{M}_t$ is retrained before the next iteration.

\begin{algorithm}[t]
\small
\caption{Graph-based {\color{softgreen}context-guided} search}
\label{alg:graph_search}
\begin{algorithmic}[1]
\REQUIRE Budget $B$, evaluation cap per iteration $n_{\text{eval}}$
\STATE Sample $a_0 \sim \mathrm{LLM}_{\text{alg}}(\cdot \mid \emptyset)$, evaluate $a_0$ to obtain $\widehat{V}(a_0)$, and initialize graph $G_0$
\STATE $t \gets 0; \quad \mathcal{A}_0^{\text{eval}} \gets \{a_0\}; \quad \mathcal{U}_0 \gets \emptyset; \quad \Phi \gets \emptyset {\color{softgreen}; \quad \mathcal{S}_0 \gets \emptyset}$
\WHILE{$B > 0$}
    \STATE $a_{\text{ref}} \gets \arg\max_{a \in \mathcal{A}_t^{\text{eval}}} \widehat{V}(a)$ \quad \textit{// Select reference algorithm}
    \STATE $\Phi_t, c_\text{cor} \gets \mathrm{LLM}_{\text{cor}}(a_{\text{ref}})$ {\color{softgreen} or $\mathrm{LLM}_{\text{cor}}(a_{\text{ref}}, \mathcal{S}_t); \quad B \gets B - c_\text{cor}$} \quad \textit{// Generate corrections}
    \FORALL{$\phi \in \Phi_t$}
        \STATE \textit{// Augment graph}
        \STATE Insert $\phi$'s sub-paths into $G_t$
        \STATE Add newly accessible paths to $\mathcal{A}_t$ and $\mathcal{U}_t$
        \STATE \textit{// Evaluate candidates}
        \STATE $\mathcal{E}_\phi \gets \left\{\phi(a) : a \in \mathrm{Top}^{\widehat{V}}_{n_{\text{eval}}/2}\left(\mathcal{A}_t^{\text{eval}} \cap D_\phi\right)\right\} \cup \mathrm{Top}^{\mathcal{M}_t}_{n_{\text{eval}}/2}\left(\mathcal{U}_t\right)$
        \STATE Evaluate $\widehat{V}(a)$ for all $a \in \mathcal{E}_\phi$
        \STATE $\mathcal{A}_{t}^{\text{eval}} \gets \mathcal{A}_t^{\text{eval}} \cup \mathcal{E}_\phi; \quad \mathcal{U}_{t} \gets \mathcal{U}_t \setminus \mathcal{E}_\phi$
        \STATE \textit{// Update surrogate}
        \STATE $\Phi \gets \Phi \cup \{\phi\}$ 
        \STATE Retrain $\mathcal{M}_{t}$ on $\left\{(r(a), \widehat{V}(a)) : a \in \mathcal{A}_{t}^{\text{eval}}\right\}$
    \ENDFOR
    {
    \color{softgreen}
    \STATE \textit{// Update search summary}
    \STATE Compute $\{\Delta_\phi\}_{\phi \in \Phi}$ from $\mathcal{M}_{t}$
    \STATE $\mathcal{S}_{t+1}, c_\text{sum} \gets \mathrm{LLM}_{\text{summary}}(\mathcal{S}_{t}, \{(\phi, \Delta_\phi)\}_{\phi \in \Phi_t}); \quad B \gets B - c_\text{sum}$
    }
    \STATE $\mathcal{A}_{t+1}^{\text{eval}} \gets \mathcal{A}_t^{\text{eval}}; \quad \mathcal{U}_{t+1} \gets \mathcal{U}_t; \quad G_{t+1} \gets G_t; \quad t \gets t+1$
\ENDWHILE
\RETURN $\arg\max_{a \in \mathcal{A}_t^{\text{eval}}} \widehat{V}(a)$
\end{algorithmic}
\end{algorithm}
 
\paragraph{Updating the surrogate model.} A central benefit of the graph representation is that it enables learning the contribution of each individual correction to algorithm fitness. Because the set $\mathcal{C}(a)$ of corrections used by algorithm $a$ is directly observable from the path of $a$, each algorithm can be mapped to a low-dimensional representation $r(a) = \left(\mathds{1}_{\{\phi \in \mathcal{C}(a)\}}\right)_{\phi \in \Phi}$. We train a random forest regressor $\mathcal{M}_t : \{0,1\}^{|\Phi|} \to \mathbb{R}$ on the dataset $\left\{(r(a), \widehat{V}(a)) : a \in \mathcal{A}_t^{\text{eval}}\right\}$ to predict fitness based on the presence of corrections in the cumulative set $\Phi$. This model is used to rank unevaluated paths and extract high-potential candidates for the next evaluation step.
 
\paragraph{Search summary (context-guided variant only).} To provide the correction-generating LLM with actionable context from past search iterations, we maintain a concise natural-language summary $\mathcal{S}_t$ that synthesizes the marginal contribution $\Delta_\phi$ of generated corrections. Estimating these contributions is complicated by noisy evaluations and correction dependencies. We let $\psi_\phi(a)$ denote the Shapley value of correction $\phi$ in the prediction $\mathcal{M}_t(a)$, and quantify the marginal contribution of $\phi$ via the difference in Shapley value between evaluated algorithms that use $\phi$ and those that do not:
\begin{equation*}
\Delta_\phi \triangleq \frac{1}{|\mathcal{A}_t^{\phi}|}\sum_{a \in \mathcal{A}_t^{\phi}}\psi_\phi(a) - \frac{1}{|\mathcal{A}_t^{\neg\phi}|}\sum_{a \in \mathcal{A}_t^{\neg\phi}}\psi_\phi(a),
\end{equation*}
where $\mathcal{A}_t^{\phi} = \{a \in \mathcal{A}_t^{\text{eval}} : \phi \in \mathcal{C}(a)\}$ and $\mathcal{A}_t^{\neg\phi} = \mathcal{A}_t^{\text{eval}} \setminus \mathcal{A}_t^{\phi}$. This estimate disentangles $\phi$'s contribution from those of co-occurring corrections and, by averaging over all evaluated algorithms, provides a robust measure of corrections' effectiveness. The new pairs $\{(\phi, \Delta_\phi)\}_{\phi \in \Phi_t}$ produced at each iteration are passed, along with the current summary $\mathcal{S}_t$, to $\textrm{LLM}_\textrm{summary}$ to produce an updated summary $\mathcal{S}_{t+1}$.

\section{Theoretical Insights}

Existing AAD methods, including our graph-based search, typically leverage multiple parallel runs to improve overall performance. In this section, we address the question: given a total budget $B$, how should one allocate this budget across individual runs? We parameterize this allocation using a per-run budget $\omega \in [c_0, B]$, capturing the search depth of individual runs, where $c_0$ is the minimum cost required to generate a single algorithm. Given $\omega$, the policy $\pi$ performs independent random restarts, providing search breadth, with each trial bounded by cost $\omega$ until the total budget $B$ is exhausted. This allows for $N(\omega) = \lfloor B/\omega \rfloor$ parallel executions of $\pi$. With $z(\pi, \omega)$ being the realized fitness of a single run, we define the expected fitness under total budget $B$ of policy $\pi$ with per-run cost cap $\omega$ as:
\begin{equation*}
    Z_B\left(\pi, \omega \right) \triangleq \mathbb{E}_{\pi, L}\left[\max_{i \in [N]} z^{(i)}(\pi, \omega)\right], \quad N = \left\lfloor \frac{B}{\omega} \right\rfloor,
\end{equation*}
where $z^{(1)}(\pi, \omega), \dots, z^{(N)}(\pi, \omega)$ are samples of $z(\pi, \omega)$. These samples are assumed i.i.d., reflecting parallel runs with independent random seeds. Writing $\mu_\pi(\omega) \triangleq \mathbb{E}\left[z(\pi, \omega)\right]$ for the per-run mean fitness, $Z_B(\pi, \omega)$ can be decomposed as:
\begin{equation*}
Z_B\left(\pi, \omega\right) = \mu_\pi(\omega) + \mathbb{E} \left[\max_{i \in [N]}\left(z^{(i)}(\pi, \omega) - \mu_\pi(\omega)\right)\right].
\end{equation*}
For analytical tractability, we consider a continuous relaxation and assume that the second term depends on $\omega$ only through the effective number of runs $B/\omega$, yielding the structural decomposition:
\begin{equation} \label{eq:decomposition}
    Z_B\left(\pi, \omega\right) = \mu_\pi(\omega) + Q_\pi \left(\frac{B}{\omega}\right).
\end{equation}
which holds when the tail of $z(\pi, \omega) - \mu_\pi(\omega)$ is independent of $\omega$. This decoupling preserves the essential structure: increasing the per-run budget $\omega$ improves the expected per-run fitness $\mu_\pi(\omega)$ at the cost of fewer runs, and thus a smaller order-statistic gain $Q_\pi(B/\omega)$. Both $\mu_\pi$ and $Q_\pi$ naturally exhibit diminishing returns in how mean performance scales with the per-run budget and in the order-statistic maxima achieved across independent runs. We formalize this by assuming that $\mu_\pi: [c_0, \infty) \to \mathbb{R}$ and $Q_\pi: [1, \infty) \to \mathbb{R}_{\geq 0}$ are twice differentiable, increasing, and concave on their domain. We also assume that $\mu_\pi(c_0)$ is identical across all policies $\pi$, since all runs reduce to sampling a single algorithm. 
 
We let $\underline{\omega}_{\pi}(B)$ and $\overline{\omega}_{\pi}(B)$ be the smallest and largest elements of $\Omega^*_{\pi}(B)$, respectively, where
\begin{equation*}
    \Omega^*_{\pi}(B) \triangleq \big\{ \omega \in [c_0, B] : \omega \text{ is a maximizer of } Z_B(\pi, \cdot) \big\}
\end{equation*}
denote the set of locally optimal per-run budgets.

\begin{theorem}[Monotonicity in $B$ and uniqueness]
\label{thm:monotonicity_B}
If $N \mapsto N Q_\pi'(N)$ is non-increasing on $[1,\infty)$, then both $\underline{\omega}_\pi$ and $\overline{\omega}_\pi$ are non-decreasing in $B$.\\
If additionally, $N \mapsto N^2 Q_\pi'(N)$ is decreasing on $[1,\infty)$, then $\Omega^*_{\pi}(B) = \{\omega^\star(B)\}$ with $\omega^\star(B)$ non-decreasing in $B$. 
\end{theorem}

Theorem \ref{thm:monotonicity_B} shows that when the distribution of $Q_\pi$ is sufficiently light-tailed, one should allocate more effort to individual runs as the total budget increases, rather than simply increasing the number of runs.
The first condition holds for common distributions, including Gaussian and uniform. Under the second, more restrictive condition on the tail of $Q_\pi$, this per-run budget is unique and non-decreasing. However, this trend toward greater depth remains tempered by the need for breadth. Our next result shows that the optimal strategy never collapses into a single exhaustive run as the total budget increases.

\begin{theorem}[Interior optimum] \label{thm:interior}
Suppose $Z_B(\pi, \cdot)$ is bounded above uniformly for any $B > 0$. Then there exists $\bar{B} > 0$ such that:
\begin{equation*}
    \Omega^*_{\pi}(B) \subseteq (c_0, B) \qquad \forall B \geq \bar{B}.
\end{equation*}
\end{theorem}

Provided the budget $B$ is sufficiently large, the optimal strategy strictly balances per-run computational effort with the quantity of independent runs. In this regime, restarts are key to ensure the budget isn't exhausted on a single trajectory that may have plateaued. While the total budget $B$ dictates the scale of the search, the inherent efficiency of the policy itself also influences how aggressively the per-run budget should be tuned.

\begin{theorem}[Monotonicity in $\mu'$]
\label{thm:monotonicity_mu_prime}
Let $\pi_1, \pi_2$ be two policy families sharing the same order-statistic term $Q_\pi$. If $\mu_{\pi_2}'(\omega) \geq \mu_{\pi_1}'(\omega)$ for all $\omega \in [c_0, B]$, then:
\begin{enumerate}
    \item $Z_B(\pi_2, \omega) \geq Z_B(\pi_1, \omega), \quad \forall \omega \in [c_0, B]$, 
    \item $\underline{\omega}_{\pi_2}(B) \geq \underline{\omega}_{\pi_1}(B)$, and $\overline{\omega}_{\pi_2}(B) \geq \overline{\omega}_{\pi_1}(B)$.
\end{enumerate}
\end{theorem}

A policy that yields greater marginal performance gains per unit of run budget performs better at all effort levels $\omega$, and the set of optimal per-run budgets of the higher marginal performance policy dominates that of the lower marginal performance policy in the product order. This suggests that efficient policies should prioritize depth over breadth.

\section{Numerical Results} \label{sec:numerical_results}

We evaluate our framework on three combinatorial optimization problems: the Traveling Salesman Problem (TSP), the Location Routing Problem (LRP)~\citep{balakrishnan1987integrated}, and a variant of the Bus Routing Problem (BRP)~\citep{bertsimas2019optimizing} without predefined stop locations. We compare our \emph{graph-based} search and its \emph{context-guided} variant against a \emph{baseline} that generates full algorithms with natural-language descriptions and refines them iteratively as in EoH~\citep{liu2024evolution}. All LLM queries use DeepSeek-V3.2 \citep{liu2025deepseek}. Implementation details, problem specifications, and prompts are provided in Appendices \ref{apx:implementation_details}, \ref{apx:problem_specifications}, and \ref{apx:prompts}, respectively.

We run each method 50 times with 20 iterations per run. All methods start from the same 50 initial, randomly seeded algorithms to reduce variance across runs. Figure~\ref{fig:fitness_cost} displays the fitness trajectories relative to cumulative token usage across all runs for TSP and LRP, highlighting the 75th, 90th, and 95th quantiles. BRP is omitted from Figure~\ref{fig:fitness_cost} for brevity; the full results are available in Appendix~\ref{apx:implementation_details}. The graph-based methods exhibit more consistent improvements across all three quantiles: whereas the baseline quickly stalls after a few iterations, the graph-based methods keep improving over a longer horizon. While the context-guided search is highly beneficial for LRP, it surprisingly degrades performance for TSP. We explain this phenomenon by the difference in prior knowledge of LLMs: TSP has been extensively studied and is heavily represented in pretraining data, providing the LLM with a strong prior \citep{masoud2024exploring, elhenawy2024eyeballing}. In contrast, LRP is comparatively underrepresented. The vanilla graph-based method already exploits this prior effectively on TSP, and providing a search summary dilutes the knowledge while incurring additional inference cost. However, on LRP, where the prior is weaker, the search substantially benefits from the summary.

\begin{figure}[ht]
    \centering
    \includegraphics[width=\textwidth]{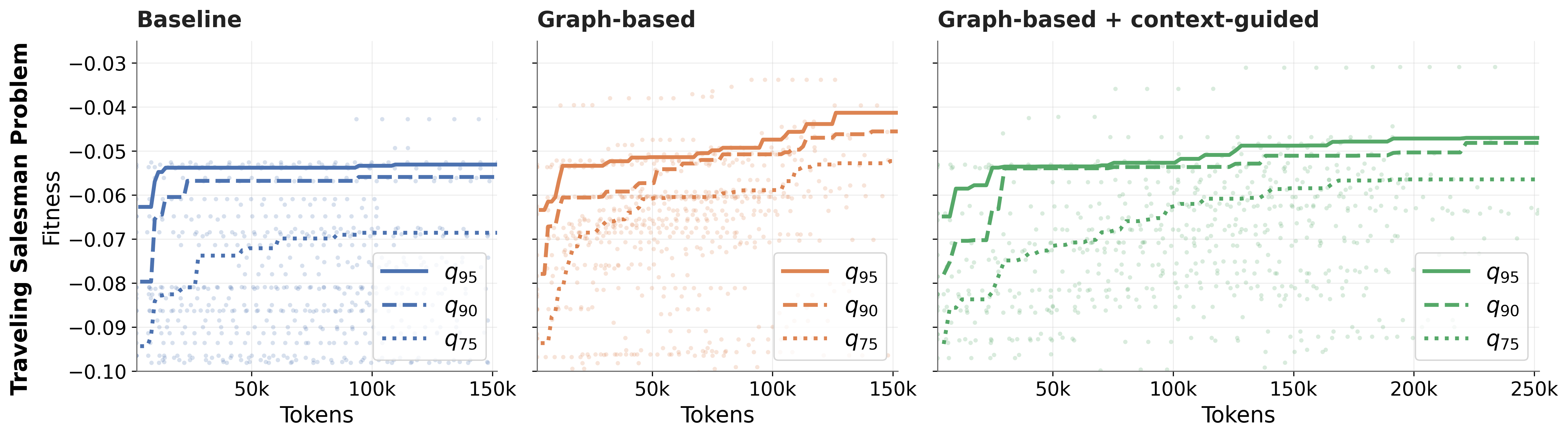}
    \includegraphics[width=\textwidth]{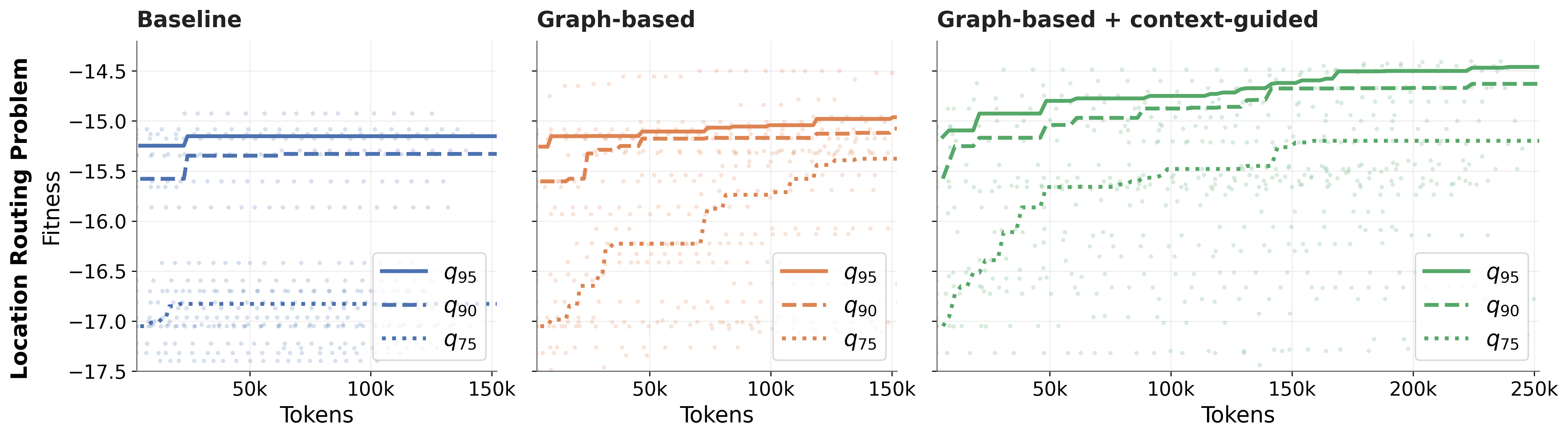}
    \caption{Fitness against cumulative tokens on TSP and LRP. The solid, dashed, and dotted lines show the 95th, 90th, and 75th quantiles, respectively.}
    \label{fig:fitness_cost}
\end{figure}


The baseline method's tendency to stall is further evidenced by the decaying improvement rates shown in Figure~\ref{fig:improvement_rate}, whereas graph-based methods demonstrate consistent gains over a longer horizon. 
We attribute this to the rapid growth in algorithm length and complexity over iterations. Refining a full, long algorithm in a single LLM query is a complex task, making the model more likely to introduce a faulty edit that degrades performance or a syntax error that renders the code unusable. Graph-based search handles long algorithms more reliably: syntax errors are less likely due to shorter outputs, and detrimental corrections can be toggled off while beneficial ones are retained.

\begin{wrapfigure}[15]{r}{0.55\textwidth} 
  \centering
  \vspace{-12pt}
  \includegraphics[width=0.43\textwidth]{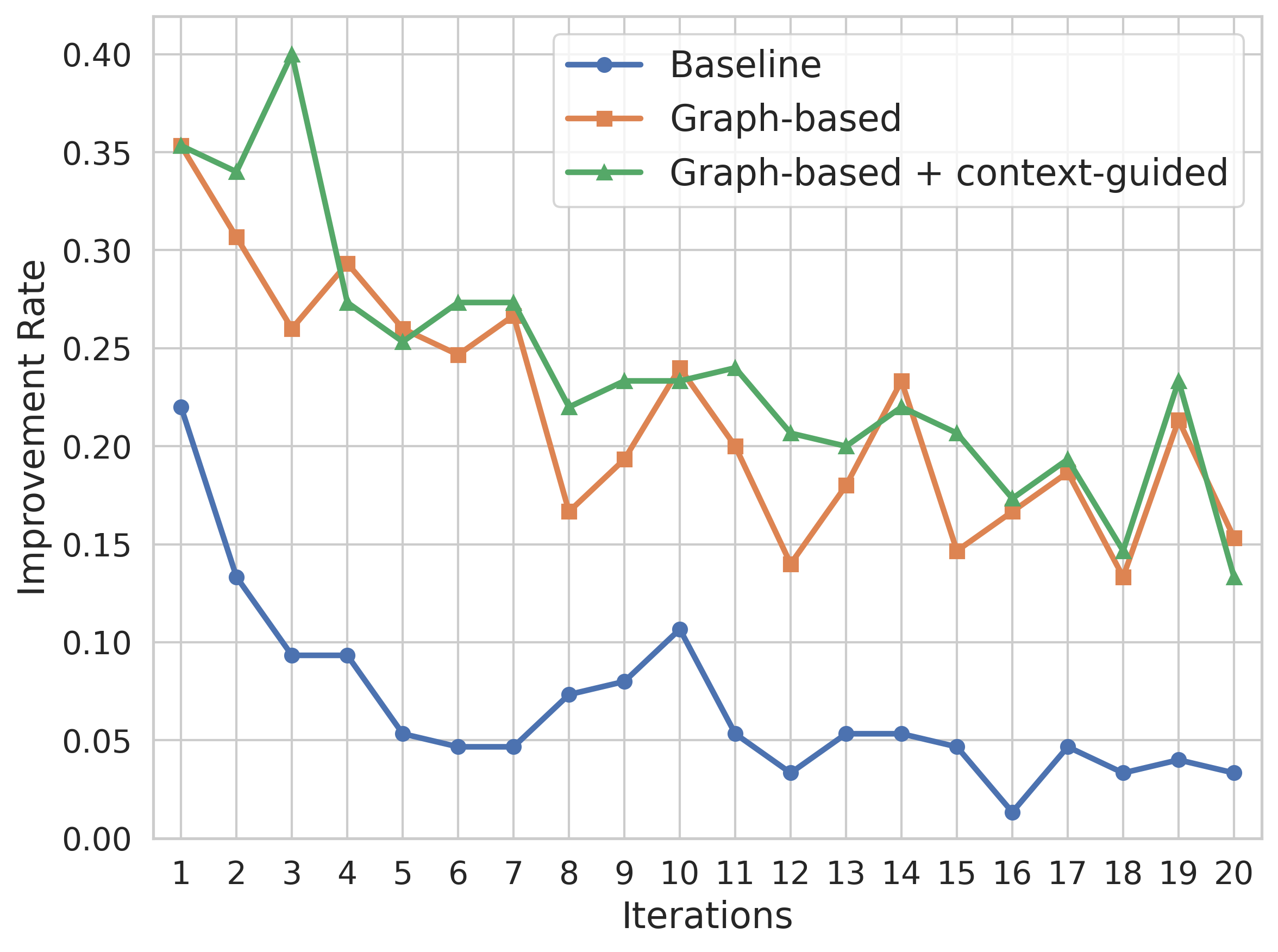}
  \caption{Proportion of iterations that strictly improve the incumbent's best fitness, averaged across all problems. Per-problem curves are shown in Appendix~\ref{apx:implementation_details}.}
  \label{fig:improvement_rate}
\end{wrapfigure}

To estimate the performance of a method capped at $n$ iterations under a budget $B$, we simulate $1000$ bootstrap trajectories of sequential rollouts. Reaching the $n$-iteration cap triggers a random restart. We track the incumbent's best fitness along the trajectory and report the final fitness averaged across the $1000$ trajectories.

We report in Table~\ref{tab:problem_budget_results} the expected best fitness achieved by each search method, capped at $n$ iterations per run, with random restarts performed until the total budget $B$ is exhausted. When $n=0$, all methods coincide with repeatedly sampling zero-shot algorithms until budget exhaustion. Three key observations emerge from these results.

First, a small budget favors multiple random restarts. Notably, at the tightest budget $B = 10^5$, no iterative refinement strategy beats zero-shot sampling for TSP. Since the intrinsic stochasticity of LLMs yields a wide diversity of zero-shot algorithms, it is more advantageous to sample multiple independent zero-shot algorithms rather than allocating a fraction of the budget on iterative refinements that yield only marginal improvements.

Second, across all problems and methods, the optimal number of iterations per run increases with the total budget $B$. This observation aligns with our theoretical findings that the per-run cost should be non-decreasing with respect to $B$ (Theorem~\ref{thm:monotonicity_B}). Similarly, the most expensive context-guided search only outperforms other methods for larger budgets and on problems for which the LLM lacks extensive knowledge.

Finally, at each method's best iterative configuration, the vanilla graph-based search outperforms the baseline across all problems and budget levels. On TSP, this corresponds to a 7.5\%, 19.7\%, and 22.3\% smaller optimality gap at $B=10^5$, $B=10^6$, and $B=10^7$, respectively, confirming the practical efficacy of our approach to fully exploit the capabilities of the LLM.

\setlength{\tabcolsep}{4.5pt}
\begin{table}[h]
\caption{Mean performance across varying budgets. For each budget, the best fitness per method is underlined, and the best fitness across all methods is bolded.}
\label{tab:problem_budget_results}
\small
\centering
\begin{tabular}{c | c | c | c c c | c c c | c c c}
\toprule
 &  &  & \multicolumn{3}{c}{Baseline} & \multicolumn{3}{c}{Graph} & \multicolumn{3}{c}{Graph + context-guided} \\
 & $B$ & n = 0 & $n=5$ & $n=10$ & $n=20$ & $n=5$ & $n=10$ & $n=20$ & $n=5$ & $n=10$ & $n=20$ \\
\midrule
\multirow{3}{*}{\rotatebox{90}{\shortstack{\textbf{TSP} \\ (x100)}}} & $10^{5}$ & \textbf{-5.476} & \underline{-7.140} & -8.217 & -9.664 & \underline{-6.603} & -6.884 & -7.779 & \underline{-7.779} & -8.202 & -9.335 \\
 & $10^{6}$ & -5.310 & \underline{-5.325} & -5.484 & -5.558 & -4.573 & -4.742 & \textbf{\underline{-4.275}} & -5.146 & -5.065 & \underline{-5.043} \\
 & $10^{7}$ & -5.310 & -5.259 & -5.261 & \underline{-4.430} & -3.958 & -3.829 & \textbf{\underline{-3.443}} & -4.606 & -3.704 & \underline{-3.651} \\
\midrule\midrule
\multirow{3}{*}{\rotatebox{90}{\textbf{LRP}}} & $10^{5}$ & \textbf{-15.16} & \underline{-16.90} & -19.73 & -24.38 & \underline{-16.50} & -19.26 & -22.31 & \underline{-16.73} & -20.19 & -20.70 \\
 & $10^{6}$ & -15.08 & \underline{-15.15} & -15.16 & -15.50 & \underline{-14.90} & -15.00 & -15.07 & \textbf{\underline{-14.70}} & -14.81 & -14.83 \\
 & $10^{7}$ & -15.08 & -15.08 & \underline{-14.93} & -14.95 & -14.63 & -14.57 & \underline{-14.53} & -14.49 & -14.50 & \textbf{\underline{-14.42}} \\
\midrule\midrule
\multirow{3}{*}{\rotatebox{90}{\textbf{BRP}}} & $10^{5}$ & -2067 & \underline{-1991} & -2033 & -2070 & \textbf{\underline{-1969}} & -2053 & -2146 & \underline{-2073} & -2109 & -2141 \\
 & $10^{6}$ & -1989 & -1732 & \underline{-1642} & -1679 & \textbf{\underline{-1580}} & -1600 & -1668 & \underline{-1583} & -1599 & -1720 \\
 & $10^{7}$ & -1989 & -1705 & \underline{-1505} & -1518 & -1528 & -1528 & \underline{-1471} & -1467 & \textbf{\underline{-1366}} & -1393 \\
\bottomrule
\end{tabular}
\end{table}

\section{Conclusion}

We introduced a graph-based search framework for budget-efficient AAD that alleviates the inefficiencies of existing LLM-driven pipelines. By representing algorithms as paths in a directed acyclic graph and querying the LLM for localized, reusable corrections rather than full rewrites, our method significantly increases algorithm yield and enables correction-level credit assignment. We demonstrated that this approach consistently outperforms full-algorithm search baselines on combinatorial optimization tasks, while also providing theoretical and empirical insights into depth-versus-breadth trade-offs under budget constraints. Specifically engineered to navigate the bottleneck of costly LLM inference, a limitation of our framework is its reliance on a cheap evaluation oracle. In domains where the oracle is highly expensive, the greater number of evaluations would likely undermine the cost savings on inference. Future work includes adaptive tuning of the per-run budget and information sharing across parallel runs.


{
\small
\bibliography{refs}
}

\newpage
\appendix
\section{Proofs}

Let $N = B/\omega$. Differentiating $Z_B$ once with respect to $\omega$:
\begin{align} \label{eq:Z_first_deriv}
    \frac{\partial Z_B}{\partial \omega}(\pi, \omega) &= \mu_\pi'(\omega) - \frac{B}{\omega^2} Q_\pi' \left(\frac{B}{\omega}\right)\\
    &=  \mu_\pi'(\omega) - \frac{N}{\omega} Q_\pi' \left(N\right) \nonumber
\end{align}
and twice, with $N = B/\omega$,
\begin{align} \label{eq:Z_second_deriv}
    \frac{\partial^2 Z_B}{\partial \omega^2}(\pi, \omega) &= \mu_\pi''(\omega) + \frac{2B}{\omega^3}Q_\pi' \left(\frac{B}{\omega}\right) + \frac{B^2}{\omega^4} Q_\pi'' \left(\frac{B}{\omega}\right)\\
    &= \mu_\pi''(\omega) + \frac{2N}{\omega^2}Q_\pi' \left(N\right) + \frac{N^2}{\omega^2} Q_\pi'' \left(N\right) \nonumber\\
    &= \mu_\pi''(\omega) + \frac{1}{\omega^2} \frac{d}{dN}\left[ N^2 Q_\pi'(N) \right] \nonumber
\end{align}

\begin{lemma}\label{lemma:bounded_limit}
    If $h: [a, \infty) \to \mathbb{R}$ is differentiable, increasing, concave, and bounded above, then $\lim_{x\to \infty} xh'(x) = 0$.
\end{lemma}

\begin{proof}
    Since $h$ is increasing and bounded above, there exists $l = \lim_{x\to \infty} h(x)$. Thus $\lim_{x\to \infty} h(2x) - h(x) = l - l = 0$. The concavity of $h$ implies that $h'$ is non-increasing so for $t \in [x, 2x]$, $h'(t) \geq h'(2x)$. Therefore $h(2x) - h(x) = \int_x^{2x} h'(t) dt \geq \int_x^{2x} h'(2x) dt = xh'(2x)$. It follows that $\lim_{x\to \infty} xh'(2x) = 0$.
\end{proof}

\subsection{Proof of Theorem~\ref{thm:monotonicity_B}}

\begin{proof}
We show that $Z_B(\pi, \omega)$ exhibits increasing differences in $(\omega, B)$. From \eqref{eq:Z_first_deriv}, we have:
$$\frac{\partial Z_B}{\partial \omega}(\pi, \omega) = \mu_\pi'(\omega) - \frac{B}{\omega^2} Q_\pi' \left(\frac{B}{\omega}\right)$$
Letting $N = B/\omega$, differentiating with respect to $B$ gives:
\begin{align*}
    \frac{\partial^2 Z_B}{\partial \omega  \partial B}(\pi, \omega) &= -\frac{1}{\omega^2} Q_\pi'\left(\frac{B}{\omega}\right) - \frac{B}{\omega^3} Q_\pi''\left(\frac{B}{\omega}\right) \\
    &= -\frac{1}{\omega^2}\left( Q_\pi'(N) + N Q_\pi''(N) \right)\\
    &= -\frac{1}{\omega^2} \frac{\partial}{\partial N}\left( N Q_\pi'(N) \right),
\end{align*}
where the last equality uses $\frac{d}{dN}[N Q_\pi'(N)] = Q_\pi'(N) + N Q_\pi''(N)$.

Since $N \mapsto N Q_\pi'(N)$ is non-increasing on $[1, \infty)$, its derivative is non-positive so $\frac{\partial^2 Z_B}{\partial \omega  \partial B}(\pi, \omega) \geq 0$ on $\{(\omega, B) : c_0 \leq \omega \leq B\}$.
Therefore, for any $B_1 \leq B_2$ and any $c_0 \leq \omega_L \leq \omega_H \leq B_1$,
\begin{equation*}
    Z_{B_2}(\pi, \omega_H) - Z_{B_2}(\pi, \omega_L) \geq Z_{B_1}(\pi, \omega_H) - Z_{B_1}(\pi, \omega_L).
\end{equation*}
By Topkis's monotone comparative statics theorem, $\Omega^\star_\pi(B) = \arg\max_{\omega \in [c_0, B]} Z_B(\pi, \omega)$ is non-decreasing in $B$ in the strong set order, so both $\underline{\omega}_\pi(B)$ and $\overline{\omega}_\pi(B)$ are non-decreasing in $B$.

Further assume that $N \mapsto N^2 Q_\pi'(N)$ is strictly decreasing on $[1, \infty)$. Let $G(N) \triangleq N^2 Q_\pi'(N)$, so $G'(N) \leq 0$. Using \eqref{eq:Z_second_deriv}, we have:
\begin{align*}
    \frac{\partial^2 Z_B}{\partial \omega^2}(\pi, \omega) &= \mu_\pi''(\omega) + \frac{1}{\omega^2} \frac{d}{dN}\left[ N^2 Q_\pi'(N) \right]\\
    &= \mu_\pi''(\omega) + \frac{1}{\omega^2} G'\left(\frac{B}{\omega}\right)    
\end{align*}
Since $\mu_\pi$ is concave and $G$ is strictly decreasing, it follows that $\frac{\partial^2 Z_B}{\partial \omega^2}(\pi, \omega) < 0$ on $[c_0, B]$. Thus $Z_B(\pi, \cdot)$ is strictly concave on $[c_0, B]$. It follows that $Z_B(\pi, \cdot)$ has a unique maximizer on the compact interval $[c_0, B]$. Letting $\omega^\star(B)$ be that maximizer, it holds $\Omega^*_\pi(B) = \{\omega^\star(B)\}$.

Since $\Omega^*_\pi(B)$ is a singleton, the non-decreasing property of $\underline{\omega}_\pi(B)$ and $\overline{\omega}_\pi(B)$ implies that $\omega^\star(B)$ is non-decreasing in $B$.
\end{proof}

\subsection{Proof of Theorem~\ref{thm:interior}}
\begin{proof}
    The uniform boundedness of $Z_B$ for all $B$ implies that $\mu_\pi$ and $Q_\pi$ are both bounded above. Lemma~\ref{lemma:bounded_limit} gives:
    \begin{align*}
        \lim_{\omega \to \infty} \omega\mu'(\omega) &= 0\\
        \lim_{N \to \infty} NQ'(N) &= 0
    \end{align*}
    First fix $\omega = c_0$. From~\eqref{eq:Z_first_deriv}, with $N_0 = B/c_0$:
    $$\frac{\partial Z_B}{\partial \omega} (\pi, c_0) = \mu_\pi'(c_0) - \frac{N_0}{c_0} Q_\pi'(N_0).$$
    Since $\mu_\pi'(c_0) > 0$ is fixed and $N_0 Q_\pi'(N_0) \to 0$ as $B \to \infty$, there exists $\bar{B}_1$ such that $\frac{\partial Z_B}{\partial \omega}(\pi, c_0) > 0$ for all $B \geq \bar{B}_1$. By continuity of $\frac{\partial Z_B}{\partial \omega}$ with respect to $\omega$, $Z_B$ is strictly increasing on an interval $[c_0, c_0+\epsilon]$ for some $\epsilon > 0$. Thus $c_0$ is not a local maximizer and $c_0 \notin \Omega^*_\pi(B)$.

    Now fix $\omega = B$. Again from~\eqref{eq:Z_first_deriv}, with $N_B = 1$,
    \begin{align*}
        \frac{\partial Z_B}{\partial \omega}(\pi, B) &= \mu_\pi'(B) - \frac{N_B}{B} Q_\pi'(N_B)\\
        &= \frac{1}{B} \left[ B \mu_\pi'(B) - Q_\pi'(N_B) \right]
    \end{align*}
    Since $B \mu_\pi'(B) \to 0$ as $B \to \infty$ and $Q_\pi'(1) > 0$ from the strict monotonicity of $Q_\pi$, we have that there exists $\bar{B}_2$ such that $\frac{\partial Z_B}{\partial \omega}(\pi, B) < 0$ for all $B \geq \bar{B}_2$. By continuity, $Z_B$ is strictly decreasing on an interval $[B- \epsilon, B]$ for some $\epsilon > 0$. Thus, for any $B \geq \bar{B}_2$, $\omega = B$ is not a local maximizer and $B \notin \Omega^*_\pi(B)$.

    Setting $\bar{B} = \max(\bar{B}_1, \bar{B}_2)$ gives $\Omega^*_\pi(B) \subseteq (c_0, B)$ for all $B \geq \bar{B}$.
\end{proof}

\subsection{Proof of Theorem~\ref{thm:monotonicity_mu_prime}}

\begin{proof}
The per-run mean fitness for policy $\pi_i$ at budget $\omega$ can be written as:
\begin{equation*}
    \mu_{\pi_i}(\omega) = \mu_{\pi_i}(c_0) + \int_{c_0}^{\omega}\mu_{\pi_i}'(x) dx.
\end{equation*}
By hypothesis, $\mu_{\pi_2}'(x) \geq \mu_{\pi_1}'(x)$ for all $x \in [c_0, B]$. Integrating from $c_0$ to $\omega$ yields:
\begin{equation*}
    \int_{c_0}^{\omega}\mu_{\pi_2}'(x) dx \geq \int_{c_0}^{\omega}\mu_{\pi_1}'(x) dx.
\end{equation*}
Adding $\mu_{\pi_2}(c_0) = \mu_{\pi_1}(c_0)$ to both sides gives:
\begin{equation*}
    \mu_{\pi_2}(\omega) \geq \mu_{\pi_1}(\omega) \qquad \forall \omega \in [c_0, B].
\end{equation*}
Since both policies share the order-statistic term $Q_\pi(B/\omega)$, adding it to both sides preserves the inequality:
\begin{equation*}
    Z_B(\pi_2, \omega) = \mu_{\pi_2}(\omega) + Q_\pi\left(\frac{B}{\omega}\right) \geq \mu_{\pi_1}(\omega) + Q_\pi\left(\frac{B}{\omega}\right) = Z_B(\pi_1, \omega),
\end{equation*}
which proves the first result.

Using \eqref{eq:Z_first_deriv}, and since both policies share $Q_\pi$, the difference of marginal returns reduces to:
\begin{equation*}
    \frac{\partial Z_B}{\partial \omega}(\pi_2, \omega) - \frac{\partial Z_B}{\partial \omega}(\pi_1, \omega) = \mu_{\pi_2}'(\omega) - \mu_{\pi_1}'(\omega) \geq 0,
\end{equation*}
so $Z_B(\pi_2, \omega) - Z_B(\pi_1, \omega)$ is non-decreasing in $\omega$. Equivalently, for any $\omega_L \leq \omega_H$ in $[c_0, B]$,
\begin{equation} \label{eq:incr_diff_mu}
    Z_B(\pi_2, \omega_H) - Z_B(\pi_2, \omega_L) \geq Z_B(\pi_1, \omega_H) - Z_B(\pi_1, \omega_L).
\end{equation}

Assume by contradiction that $\overline{\omega}_{\pi_2}(B) < \overline{\omega}_{\pi_1}(B)$. Apply~\eqref{eq:incr_diff_mu} with $\omega_L = \overline{\omega}_{\pi_2}(B)$ and $\omega_H = \overline{\omega}_{\pi_1}(B)$:
\begin{equation*}
    Z_B(\pi_2, \overline{\omega}_{\pi_1}(B)) - Z_B(\pi_2, \overline{\omega}_{\pi_2}(B)) \geq Z_B(\pi_1, \overline{\omega}_{\pi_1}(B)) - Z_B(\pi_1, \overline{\omega}_{\pi_2}(B)).
\end{equation*}
Since $\overline{\omega}_{\pi_1}(B)$ is a maximizer of $Z_B(\pi_1, \cdot)$, the right-hand side is non-negative, and it follows that:
\begin{equation*}
    Z_B(\pi_2, \overline{\omega}_{\pi_1}(B)) \geq Z_B(\pi_2, \overline{\omega}_{\pi_2}(B)).
\end{equation*}
Then either the equality holds, in which case $\overline{\omega}_{\pi_1}(B)$ is also a maximizer of $Z_B(\pi_2, \cdot)$ contradicting that $\overline{\omega}_{\pi_2}(B)$ is the larger maximizer of $Z_B(\pi_2, \cdot)$, or the inequality is strict, in which case $\overline{\omega}_{\pi_2}(B)$ is not a maximizer of $Z_B(\pi_2, \cdot)$, contradicting its definition. Both cases yield a contradiction, so $\overline{\omega}_{\pi_2}(B) \geq \overline{\omega}_{\pi_1}(B)$.

With a dual argument, assume by contradiction that $\underline{\omega}_{\pi_2}(B) < \underline{\omega}_{\pi_1}(B)$. Apply~\eqref{eq:incr_diff_mu} with $\omega_L = \underline{\omega}_{\pi_2}(B)$ and $\omega_H = \underline{\omega}_{\pi_1}(B)$:
\begin{equation*}
    Z_B(\pi_2, \underline{\omega}_{\pi_1}(B)) - Z_B(\pi_2, \underline{\omega}_{\pi_2}(B)) \geq Z_B(\pi_1, \underline{\omega}_{\pi_1}(B)) - Z_B(\pi_1, \underline{\omega}_{\pi_2}(B)).
\end{equation*}
Since $\underline{\omega}_{\pi_2}(B)$ is a maximizer of $Z_B(\pi_2, \cdot)$, the left-hand side is non-positive, hence:
\begin{equation*}
    Z_B(\pi_1, \underline{\omega}_{\pi_1}(B)) \leq Z_B(\pi_1, \underline{\omega}_{\pi_2}(B)).
\end{equation*}
Either the equality holds, in which case $\underline{\omega}_{\pi_2}(B)$ is also a maximizer of $Z_B(\pi_1, \cdot)$ contradicting that $\underline{\omega}_{\pi_1}(B)$ is the smallest maximizer of $Z_B(\pi_1, \cdot)$, or the inequality is strict, in which case $\underline{\omega}_{\pi_1}(B)$ is not a maximizer of $Z_B(\pi_1, \cdot)$, contradicting its definition. Both cases yield a contradiction, so $\underline{\omega}_{\pi_2}(B) \geq \underline{\omega}_{\pi_1}(B)$.
\end{proof}

\section{Implementation Details \& Additional Numerical Results} \label{apx:implementation_details}

\subsection{Implementation Details}

We use DeepSeek-V3.2 \citep{liu2025deepseek} for every query: initial algorithm generation, full-algorithm update, correction generation, and summary update. The cost of a run is its total number of input and output tokens. The temperature is fixed at the default value of 1.0. The number of evaluations per correction is capped at $n_\text{eval} = 16$ for both graph-based variants. For all problems, the runtime limit of algorithms is set to 60 seconds per instance. We use a Random Forest Regression model as the surrogate. Runs are parallelized on an AMD EPYC 9734 processor. 

\paragraph{DAG Representation of the Algorithmic Space.}
We represent every algorithm as a source-to-sink path in a directed acyclic graph (DAG) whose nodes are individual lines of code, so that each source-to-sink path maps to exactly one algorithm. Figure~\ref{fig:codetree_diagram} depicts such DAG, in which linear segments of the graph are collapsed into blocks for readability. Each correction augments the DAG by adding sub-paths that do not introduce cycles. The DAG thus acts as a compact representation of a combinatorial family of algorithms: $k$ independent corrections yield up to $2^{k}$ distinct paths, each corresponding to a syntactically valid algorithm. The DAG is initialized by connecting the source to the sink with a single edge. This edge encapsulates the complete source code of the initial algorithm, $a_0$. In Figure~\ref{fig:codetree_diagram}, the two corrections, green and purple, already induce four algorithms, namely the paths $(1\text{-}2\text{-}3\text{-}4)$, $(1\text{-}5\text{-}2\text{-}3)$, $(1\text{-}6\text{-}4)$ and $(1\text{-}5\text{-}6)$. This structure is what makes the search tractable: we store a single DAG rather than an explicit list of variants. Furthermore, we can efficiently enumerate all algorithms that include a given correction by joining its endpoints to the source and the sink.

\begin{figure}[h!]
    \centering
    \includegraphics[width=0.85\textwidth]{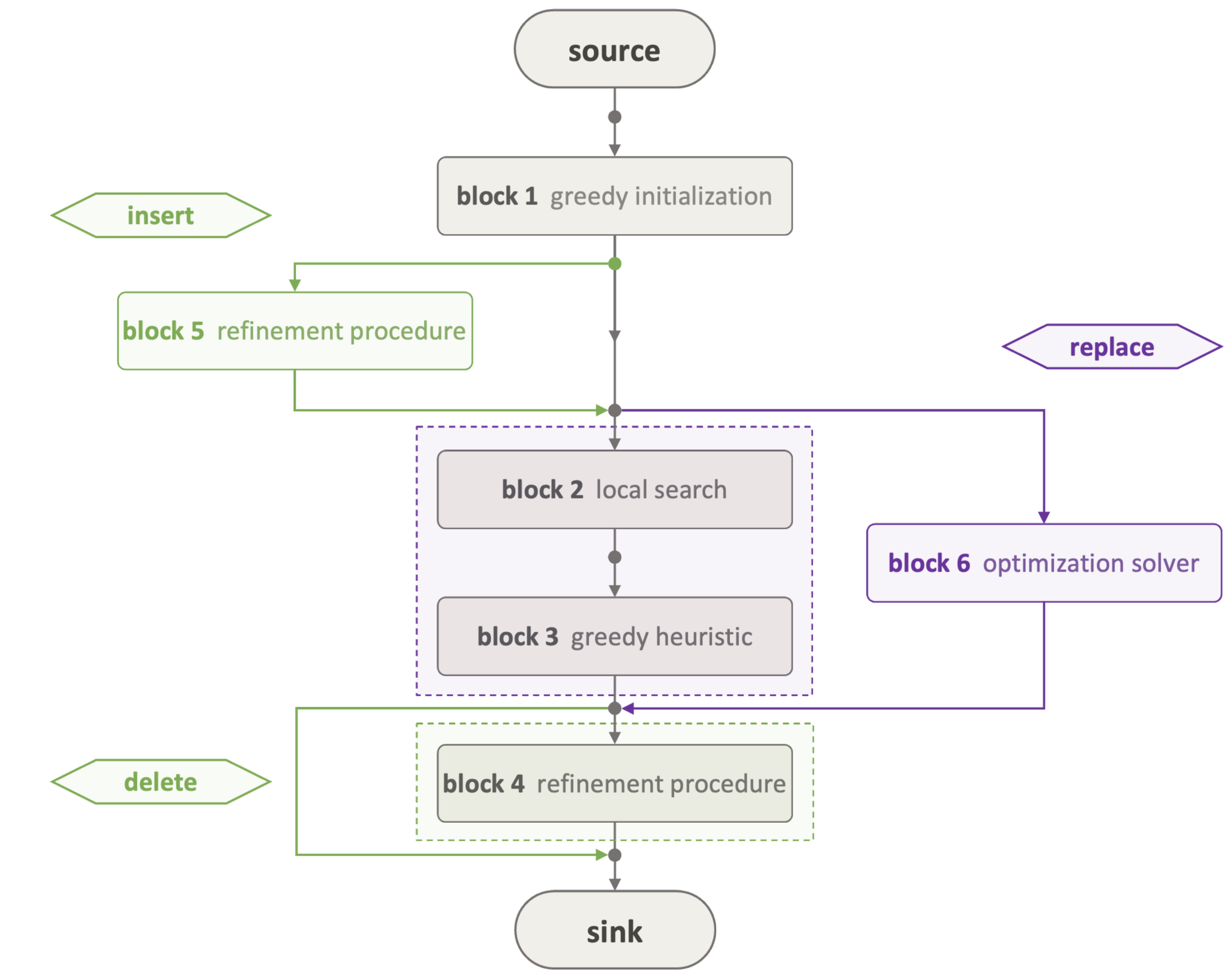}
    \caption{Augmentation of a DAG with two corrections. The green correction, ``Move the refinement procedure before the main optimization loop", consists of two edits: insert and delete. The purple correction, ``Replace the main optimization loop with Gurobi model" consists of a single replace edit.}
    \label{fig:codetree_diagram}
\end{figure}

\paragraph{Correction handling.} Recall that a correction captures a semantic modification of the reference algorithm, whereas an edit captures the underlying syntactic code modifications required to realize it. A single correction is therefore represented as a set of edits, each acting on a localized region of the code.

Corrections are produced by calling an LLM with a reference algorithm annotated with line numbers, and potentially additional context. The LLM returns a list of corrections. Each correction is composed of multiple edits characterized by (i) the type of edit ($\texttt{insert\_before}$, $\texttt{insert\_after}$, $\texttt{replace}$, $\texttt{delete}$), (ii) a couple $(n_{\text{start}}, n_{\text{end}})$, the line numbers of the first and last anchor lines, and (iii) a sequence of characters $\texttt{new\_lines}$, the code to insert or replace (ignored for the $\texttt{delete}$ operation).

Each edit is parsed as follows. First, the anchors are translated to existing nodes in the DAG based on the reference algorithm and anchor lines returned. Second, the $\texttt{new\_lines}$ are converted into a chain of nodes, which we then connect to the DAG via the first and last nodes of the anchors. A correction is then stored as a list of edits and a short natural-language description.

\paragraph{Surrogate-guided evaluation and interpretation.} Enumerating paths through the DAG is cheap, but evaluating a candidate algorithm end-to-end on the target task is not. We therefore couple the DAG with a surrogate model, a random forest that predicts the score of algorithms, trained on the candidates that have already been evaluated. The surrogate serves two distinct purposes.
\begin{enumerate}
    \item \emph{Frontier selection.} Among the exponentially many unevaluated paths, the surrogate provides a cheap ranking that lets us avoid spending compute on low-potential candidates and only consider the top-scoring ones for evaluation.
    \item \emph{Fast and robust identification of corrections' influence.} Shapley values distribute credit fairly among interacting corrections, but computing them exactly is exponential in the number of corrections. Leveraging the specific structure of trees, the methods of \cite{lundberg2018consistent} and \cite{yang2021fast} make this identification tractable.
\end{enumerate}

\subsection{Additional Numerical Results}

In Figure \ref{fig:three_contours}, we present the average performance of the baseline and both graph-based variants across TSP, LRP, and BRP. Each contour plot illustrates the performance of each method relative to varying token budgets and iterations per run. The results for BRP align with those observed for LRP: the vanilla graph-based approach consistently outperforms the baseline across all token budgets and iteration counts, and the context-guided graph-based variant achieves superior performance over the other two methods only under larger budgets.

\begin{figure}[h]
\captionsetup[subfigure]{justification=raggedright, singlelinecheck=false}
    \centering
    \begin{subfigure}{0.9\textwidth}
        \includegraphics[width=\textwidth]{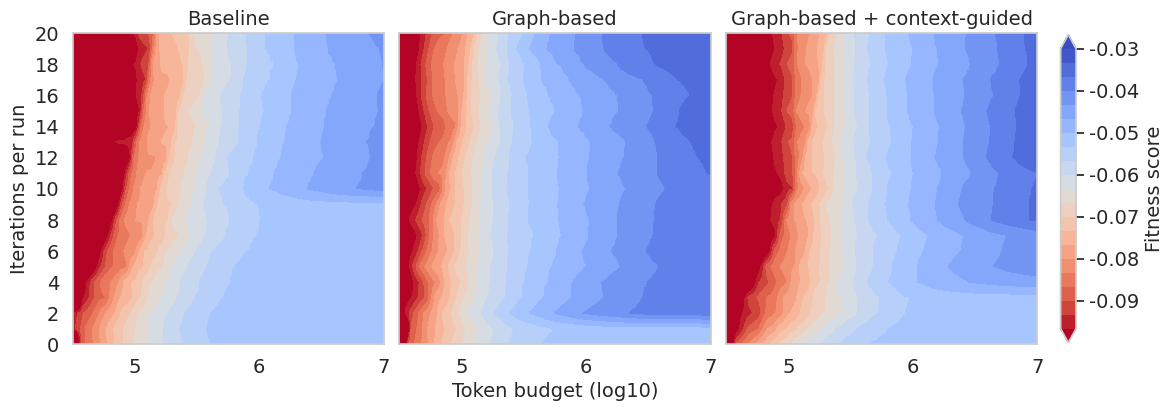}
    \end{subfigure}
    
    \begin{subfigure}{0.9\textwidth}
        \includegraphics[width=\textwidth]{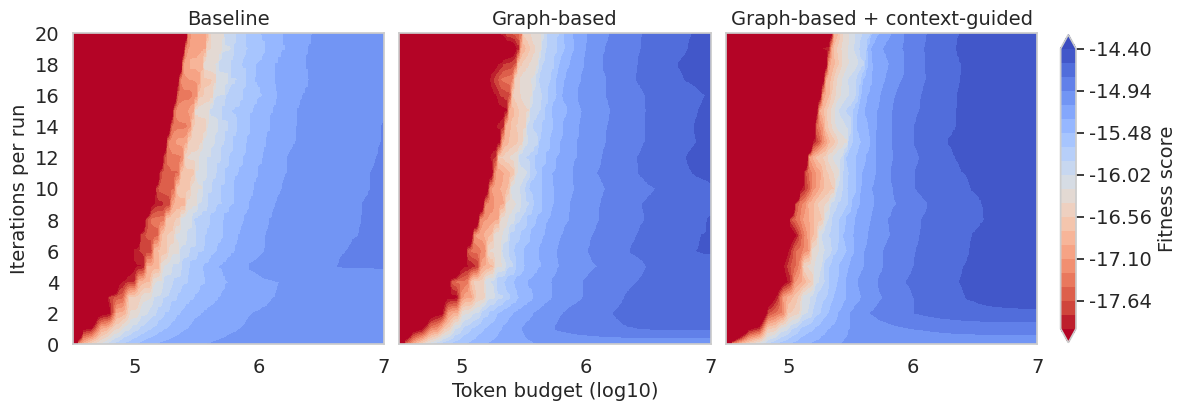}
    \end{subfigure}

    \begin{subfigure}{0.9\textwidth}
        \includegraphics[width=\textwidth]{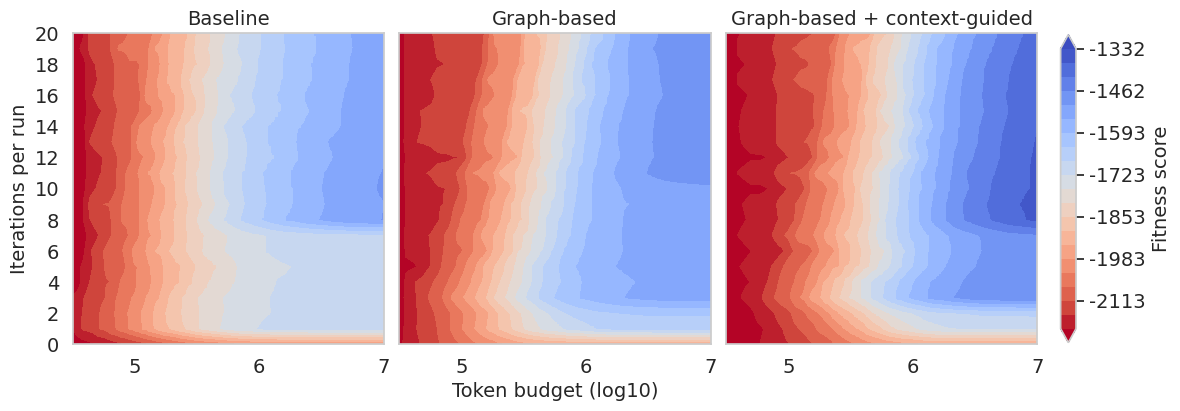}
    \end{subfigure}
    
    \caption{Contour plots of the Traveling Salesman Problem (top), Location Routing Problem (middle), and Bus Routing Problem (bottom).}
    \label{fig:three_contours}
\end{figure}

Figure \ref{fig:three_curves} reports the improvement rate (the proportion of iterations that strictly improve the incumbent's best fitness) for each problem.

\begin{figure}[ht]
    \centering
    \begin{subfigure}{0.48\textwidth}
        \centering
        \includegraphics[width=\textwidth]{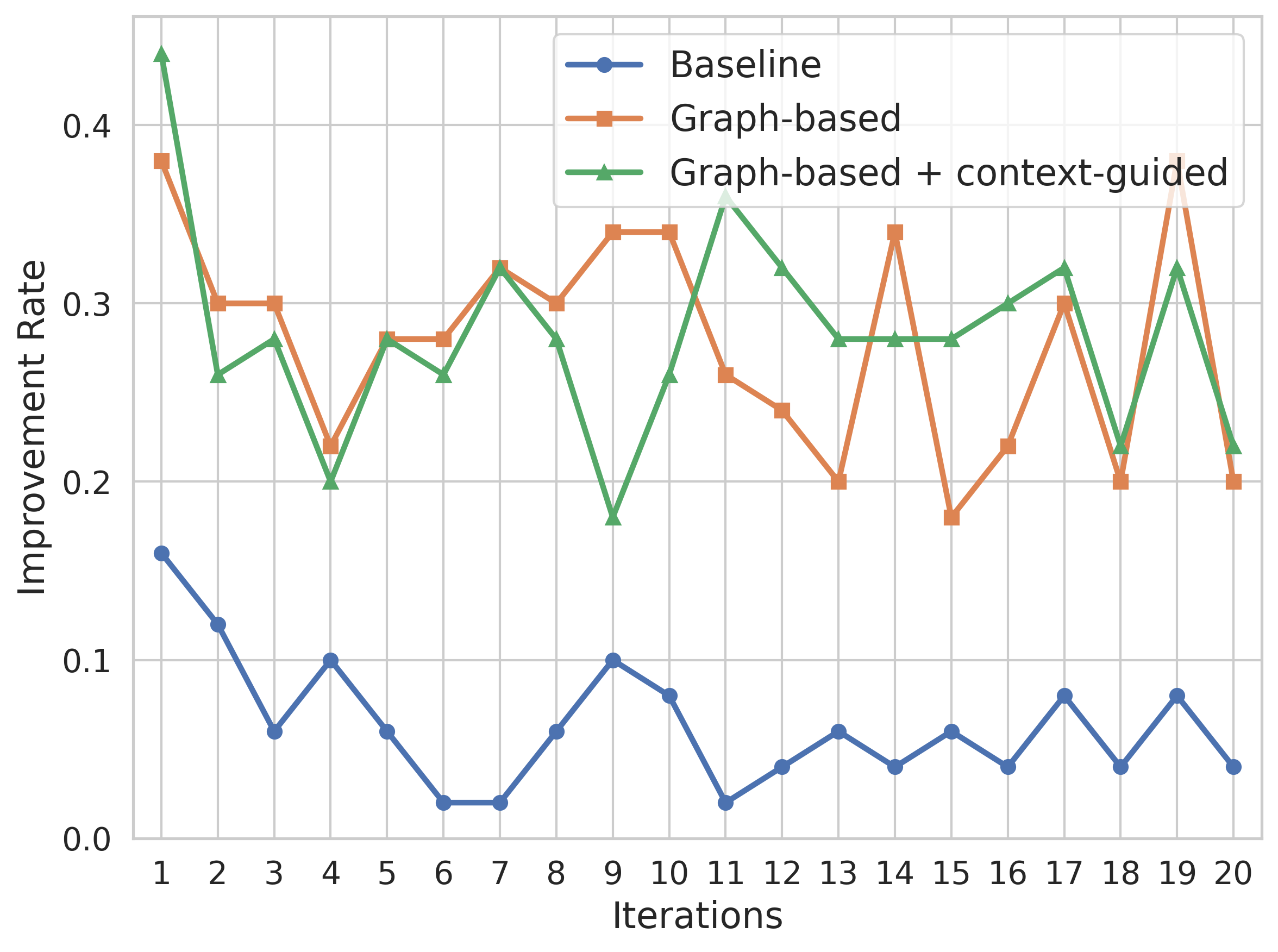}
        \caption{TSP}
    \end{subfigure}
    \hfill 
    \begin{subfigure}{0.48\textwidth}
        \centering
        \includegraphics[width=\textwidth]{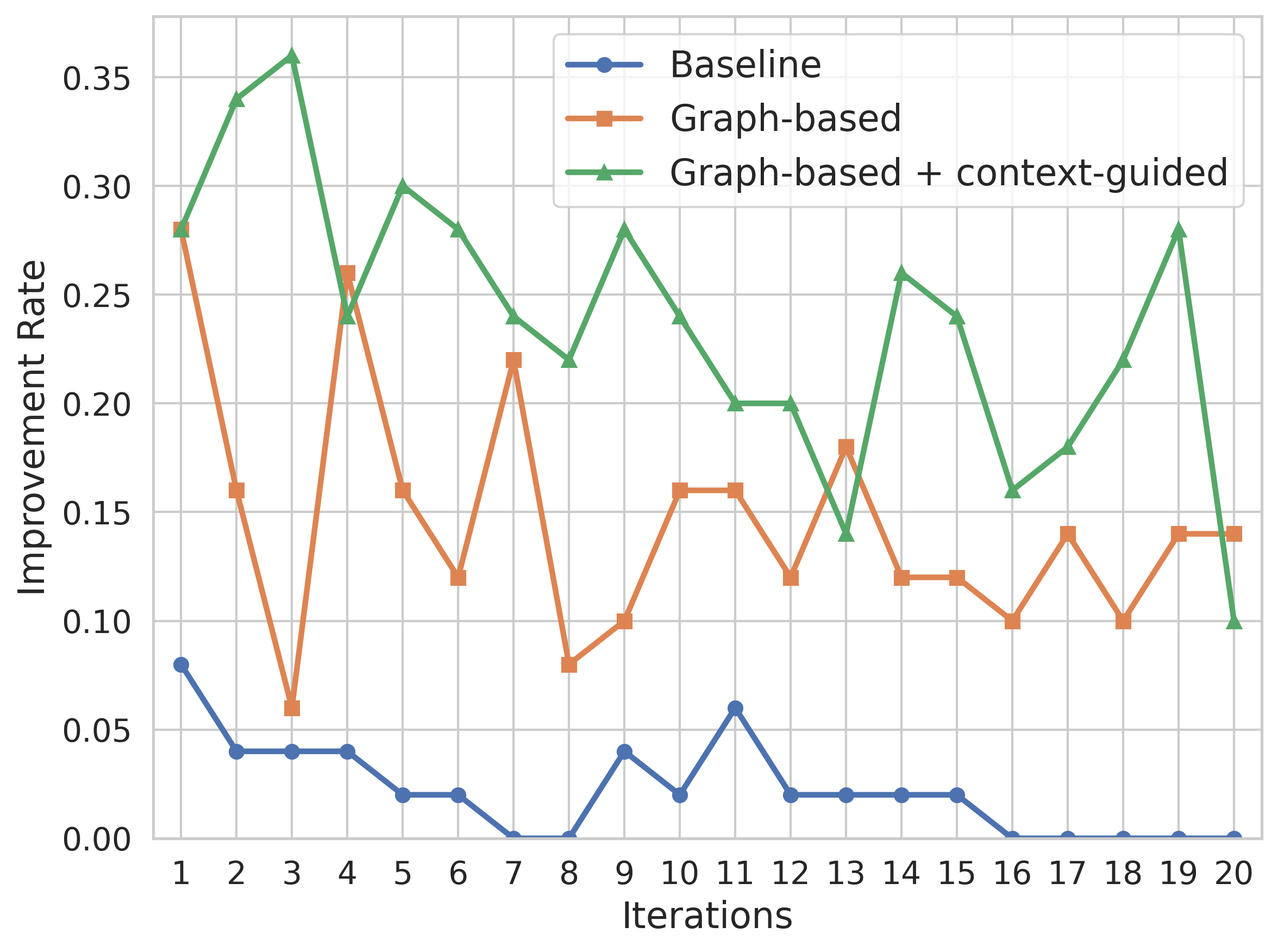}
        \caption{LRP}
    \end{subfigure}
    \begin{subfigure}{0.48\textwidth}
        \centering
        \includegraphics[width=\textwidth]{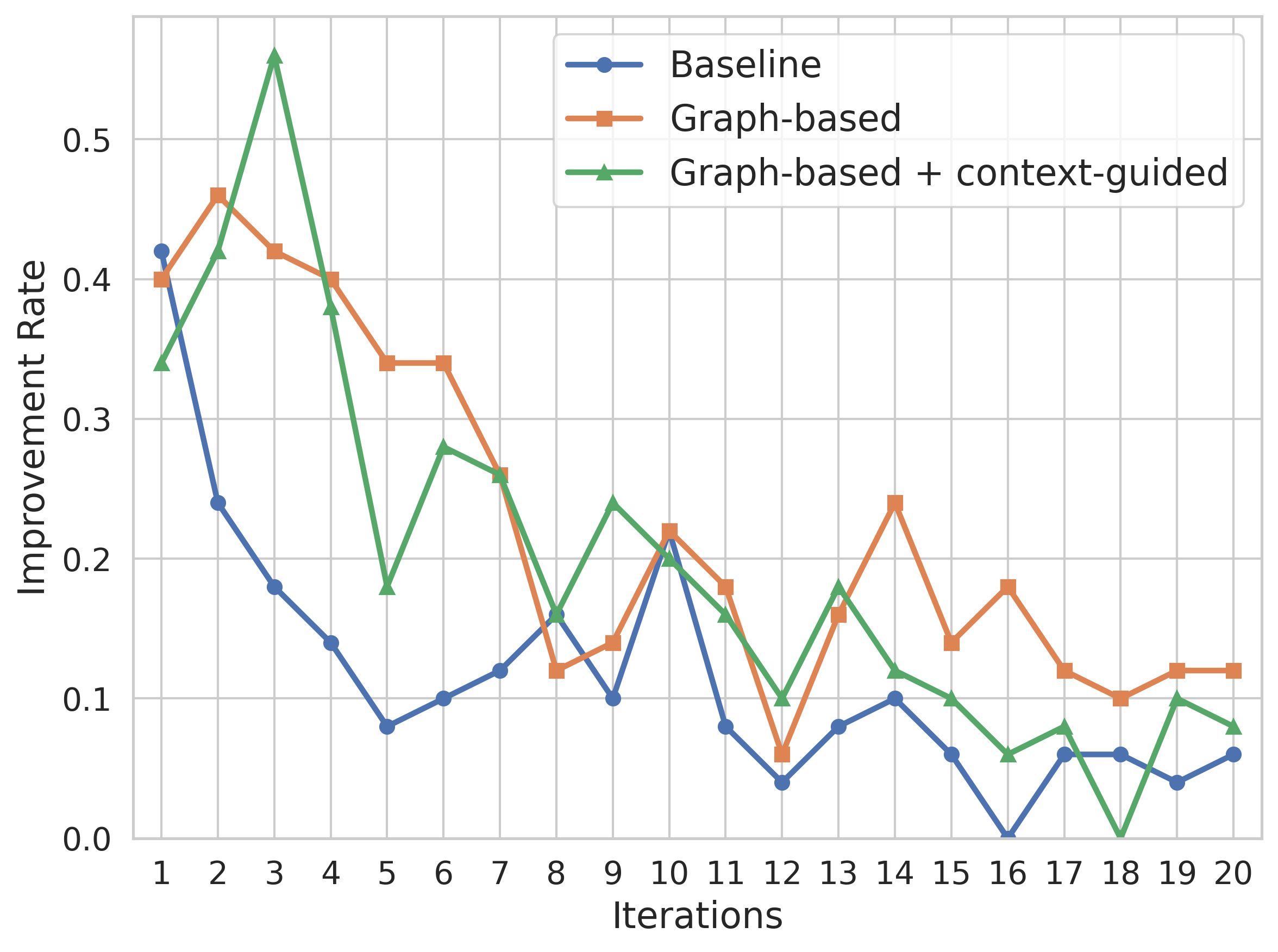}
        \caption{BRP}
    \end{subfigure}
    
    \caption{Proportion of iterations that strictly improve the
incumbent’s best fitness for each problem.}
    \label{fig:three_curves}
\end{figure}

\newpage

\clearpage
\section{Problem Specifications} \label{apx:problem_specifications}

For TSP, the fitness is the negative optimality gap averaged over 5 instances of TSPLIB \citep{reinelt1991tsplib}: \texttt{d198}, \texttt{d2103}, \texttt{fl3795}, \texttt{nrw1379}, \texttt{pcb1173}. For LRP and BRP, the fitness is the negative objective value averaged over 5 random instances of similar size.

\subsection{Traveling Salesman Problem}

\begin{promptbox}[blue]{Description}
Task:
* Given a set of nodes with their coordinates, find the shortest route that visits each node once and returns to the starting node.
* For any two vertices u, v in C, the cost c\_uv is defined as their Euclidean distance.
* The route must be closed i.e. route[0] = route[-1].
\end{promptbox}
\begin{promptbox}[blue]{Input template}
instance = {
    'customers': [  # list[dict]
        {
            'id': None,     # str: unique customer identifier (different from depot ids)
            'position': (None, None)    # tuple(float, float): 2D coordinates (x, y) of the customer
        },
        # ... more customers
    ],
    'time_limit': None  # float: time limit in seconds
}
\end{promptbox}
\begin{promptbox}[blue]{Output template}
{
  "route": None,  # list[str] e.g. ["c5", "c1", "c2", "d5"]
  "objective_value": None   # float: total cost
}
\end{promptbox}

\subsection{Location Routing Problem}

\begin{promptbox}[blue]{Description}
Given:
* A set of customers C with 2-dimensional coordinates, and demand d_i.
* A set of potential depot sites D with 2-dimensional coordinates and a fixed opening cost p_j > 0.
* A homogeneous fleet of vehicles, each with capacity Q (unlimited number of vehicles available at each open depot).
* For any two vertices u, v, the cost of traveling edge uv is c_uv is defined as their Euclidean distance.

Task:
* Choose a subset of depots D_open \subseteq D to open.
* Design a set of vehicle routes, each route starts at an open depot, visits a non-empty sequence of customers, and returns at the same depot.
* Each customer must be reached once by exactly one vehicle.
* The total demand on any single route cannot exceed Q.
* Minimize the total cost = sum of fixed opening costs of depots in D_open + sum of travel costs over all routes.

Notes:
* A feasible solution is guaranteed to exist.
\end{promptbox}

\begin{promptbox}[blue]{Input template}
instance = {
    'depots': [  # list[dict] - set D of potential depot sites
        {
            'id': None,               # str - unique depot identifier
            'opening_cost': None,     # float >= 0 - fixed cost p_j to open this depot
            'position': (None, None)  # tuple(float, float) - 2D coordinates (x, y) of the depot
        },
        # ... more depots
    ],
    'customers': [  # list[dict] - set C of customers
        {
            'id': None,             # str - unique customer identifier (different from depot ids)
            'demand': None,         # float > 0 - customer demand
            'position': (None, None)  # tuple(float, float) - 2D coordinates (x, y) of the customer
        },
        # ... more customers
    ],
    'vehicle_capacity': None,   # float > 0 - capacity Q of each (homogeneous) vehicle,
    'time_limit': None          # float: time limit in seconds
}

\end{promptbox}

\begin{promptbox}[blue]{Output template}
{
  "open_depots": [  # list[str] - ids of depots j in D that are opened
    # e.g. "d0", "d1", "d5"
  ],
  "routes": [
    ["d2", "c1", "c2", "d2"],  # example route: depot d2 -> customers c1, c2 -> depot d2
    # ... more routes
  ],
  "objective_value": None      # float - total cost = total_opening_cost + total_travel_cost
}

\end{promptbox}

\subsection{Bus Routing Problem}

\begin{promptbox}[blue]{Description}
Task:
* Given a set of students and a single, unique school, the decision maker must simultaneously determine the continuous coordinate locations for bus stops, assign students to those stops, and design bus routes, where each route may visit an ordered sequence of multiple stops before terminating at the school.
* There are no predefined candidate locations for the stops; they can be placed anywhere in the 2D coordinate space.
* The objective is to minimize the total travel time for all students: \sum_{student} (d_walk(student) / v_walk + d_bus(student) / v_bus).
* d_walk is defined as the Euclidean distance from the student's location to their assigned bus stop.
* d_bus is defined as the total Euclidean distance the bus travels from the moment a student boards at their assigned stop, through any subsequent stops on that same route, until it finally arrives at the school.
* The total number of bus routes generated must not exceed `bus_limit`.
\end{promptbox}

\begin{promptbox}[blue]{Input template}
instance = {
    'school': {
        'id': None,                 # str: unique school identifier
        'position': (None, None)    # tuple(float, float): 2D coordinates (x, y)
    },
    'students': [                   # list[dict]
        {
            'id': None,             # str: unique student identifier
            'position': (None, None)# tuple(float, float): 2D coordinates (x, y)
        },
        # ... more students
    ],
    'v_walk': None,                 # float: student walking speed
    'v_bus': None,                  # float: bus traveling speed
    'bus_limit': None,              # int: maximum number of buses available to use
    'time_limit': None              # float: time limit in seconds
}
\end{promptbox}

\begin{promptbox}[blue]{Output template}
{
  "stops": None,                # list[dict]: generated stops, e.g., [{"id": "stop_0", "position": (x, y)}, ...]
  "student_assignments": None,  # dict[str, str]: mapping of student ID to a generated stop ID
  "routes": None,               # list[list[str]]: list of routes, where each route is a sequence of generated stop IDs ending at the school ID
  "objective_value": None       # float: total travel time (walking + bus)
}
\end{promptbox}

\section{Prompts}  \label{apx:prompts}

\begin{promptbox}[gray]{Initial algorithm generation}
You are an expert algorithms engineer. Write a function solve(input) -> output that solves the following problem:
{problem_description}
--------------------

Input format:
{input_template}

--------------------

Output format:
{output_template}

--------------------

In the `thoughts` field, explain your approach in a few steps using your own terms before coding.
In the `code` field, write the complete, executable and well indented code of the 'solve' function, including necessary imports.
The code should be modular and readable, include minimal comments.
Ensure that the solve returns the *best possible solution making full use of the time budget*.
In the `description` field, describe the algorithmic approach (abstracted from the problem) as a sequence of 10-word max phrases, separated by semicolons. The description should not exceed 50 words.
\end{promptbox}

\begin{promptbox}[gray]{Full algorithm upgrade}
You are an expert algorithms engineer. Consider the following problem:
{problem_description}

--------------------

Input format:
{input_template}

--------------------

Output format:
{output_template}

--------------------

You are given the following code:
---BEGIN_CODE---
{code}
---END_CODE---
The description of the algorithm is the following:
{description}

--------------------

Your task is to implement a new algorithm that improves the performance so that it returns the *best possible solution making full use of the time budget*.
Explain your approach in a few steps using your own terms before coding.
Write the complete, executable and well indented code of the 'solve' function, including necessary imports.
The code should be modular and readable, include minimal comments.
Give the full description of the algorithmic approach (abstracted from the problem) as a sequence of 10-word max phrases, separated by semicolons. The description should not exceed 50 words.
\end{promptbox}

\begin{promptbox}[gray]{Correction generation}
You are an expert algorithms engineer. Consider the following problem:
{problem_description}

--------------------

Input format:
{input_template}

--------------------

Output format:
{output_template}

--------------------

You are given the following algorithm:
---BEGIN_ALGORITHM---
{lined_code}
---END_ALGORITHM---

--------------------

{summary} # context-guided only

--------------------

Your task is to propose multiple performance corrections to improve the fitness score of the algorithm.
Corrections are independent and self-consistent (i.e. does not rely on other corrections to execute).
Each correction consists of a short description and one or more edits (each edit is a single operation applied to one contiguous block of code).
If a correction introduces a new mechanism (e.g., variable/function/data structure), it must also include the edit(s) that integrate it; do not split prerequisites into separate corrections.
For example, an edit may introduce a new variable, and a later edit may replace a contiguous block to make the algorithm use it.
For each edit, provide an anchor by giving the numbers of the first (inclusive) and last line (exclusive) of the affected contiguous block and 'new_lines' (code to introduce).
Assume all edits are applied simultaneously. Base every anchor strictly on the unmodified input code.
Available edit operations:
- insert_after: Adds new_lines after the anchor
- insert_before: Adds new_lines before the anchor
- replace: Replaces the anchor with new_lines
- delete: Removes the anchor (new_lines should be ignored or empty)
For insert_after and insert_before, the anchor must be a single line (first == last).
For replace and delete, the anchor must be a contiguous range (first <= last).
The indentation of every line in new_lines must match the surrounding code so that, when inserted at the anchor location, the file remains syntactically valid and compiles without indentation errors.
Each line must end with a linebreak.
new_lines must not include docstrings or line numbers (#<num> lines are not part of the code and must never appear in new_lines).  
\end{promptbox}

\begin{promptbox}[gray]{Summary update}
Consider the following problem:
{problem_description}

--------------------

Input format:
{input_template}

--------------------

Output format:
{output_template}

--------------------

You are given the new corrections since the previous summary (\Delta scores if correction is valid, failure reasons if not).
* \Delta>0 means the correction improved the algorithm.
* \Delta<0 means the correction degraded the algorithm.
* \Delta=0 means the correction had no effect on the final score. Do not treat \Delta=0 as a success or a failure.

{corrections_and_delta}

--------------------

Here is the summary of the corrections of previous iterations.
{summary}

--------------------

Task 1 (`meta_analysis` field): write a meta-analysis explaining why specific traits succeeded or failed in this environment, and what are the key drivers of performance.
Task 2 (`summary` field): write the full, updated summary that will guide the correction-generating LLM toward effective edits.
Keep the previous summary as the baseline and apply the smallest set of edits needed to reflect the new corrections.
Only remove prior information if new evidence refutes it, or if the word cap leaves no room, in which case drop the entry with the weakest evidence first.
Critical rules for high-signal summary:
1. Strict grounding: only extract insights explicitly present in the provided correction history. Do not include general ML theory or textbook advice.
2. Conciseness: the summary should only include the most informative and actionable points.
3. No ambiguity: avoid vague terms.
4. Advisory tone: give strong guidelines rather than absolute rules, allowing the generator to revisit concepts if it can fix the underlying bug.
Keep the whole summary under 500 words.
Strict formatting of the summary (if the provided history contains no supporting data for a category, output '* None' under that header):
### Strongly Discouraged
* [Concept] - [concrete reason it failed]
### Highly Encouraged
* [Concept] - [concrete reason it succeeded]
### Needs Careful Implementation
* [Concept] - [presumed bug or flaw]
\end{promptbox}


%


\end{document}